\documentclass[11pt,twoside]{article}
\usepackage{xspace,fullpage,amsmath,amsthm,amssymb,amsfonts,boxedminipage,microtype,hyperref}
\usepackage{paralist}
\usepackage{bm}
\usepackage{bbm}
\usepackage{array}
\usepackage{multirow}
\usepackage{times}
\usepackage{vfmacros}

\usepackage{color}
\usepackage[style=alphabetic,backend=bibtex,maxbibnames=20,maxcitenames=6,firstinits=true,doi=false,url=false]{biblatex}
\newcommand*{\citet}[1]{\AtNextCite{\AtEachCitekey{\defcounter{maxnames}{2}}} \textcite{#1}}

\newcommand*{\citep}[1]{\cite{#1}}

\bibliography{vf-allrefs-central}

\newtheorem{remark}[thm]{Remark}
\newcommand{\bx}{{\mathbf x}}

\newcommand{\bz}{{\mathbf z}}

\newcommand{\ind}[1]{{\mathbf 1}_{\{#1\}}}
\newcommand{\STAT}{\mbox{STAT}}
\newcommand{\VSTAT}{\mbox{VSTAT}}

\newcommand{\COMM}{\mbox{1-STAT}}

\title{Dealing with Range Anxiety in Mean Estimation via Statistical Queries}

\author{Vitaly Feldman \\
IBM Research - Almaden
}

\begin{document}

\date{}

\maketitle

\begin{abstract}
We give algorithms for estimating the expectation of a given real-valued function $\phi:X\to \R$ on a sample drawn randomly from some unknown distribution $D$ over domain $X$, namely $\E_{\bx\sim D}[\phi(\bx)]$. Our algorithms work in two well-studied models of restricted access to data samples. The first one is the statistical query (SQ) model in which an algorithm has access to an {\em SQ oracle} for the input distribution $D$ over $X$ instead of i.i.d.~samples from $D$. Given a query function $\phi:X \rar [0,1]$, the oracle returns an estimate of $\E_{\bx\sim D}[\phi(\bx)]$ within some tolerance $\tau$. The second, is a  model in which only a single bit is communicated from each sample. In both of these models the error obtained using a naive implementation would scale polynomially with the range of the random variable $\phi(\bx)$ (which might even be infinite). In contrast, without restrictions on access to data the expected error scales with the standard deviation of $\phi(\bx)$. Here we give a simple algorithm whose error scales linearly in standard deviation of $\phi(\bx)$ and logarithmically with an upper bound on the second moment of $\phi(\bx)$.

As corollaries, we obtain algorithms for high dimensional mean estimation and stochastic convex optimization in these models that work in more general settings than previously known solutions.

\end{abstract}
\section{Overview}
We consider the problem of estimating the expectation $D[\phi] \doteq \E_{\bx\sim D}[\phi(\bx)]$, where $D$ is the unknown input distribution over $X$ and  $\phi:X \rar \R$. We study this problem in the statistical query (SQ) model in which an algorithm has access to a statistical query oracle for $D$ in place of i.i.d.~samples from $D$ as in the traditional setting of statistics and machine learning. The most commonly studied SQ oracle was introduced by \citet{Kearns:98} and gives an estimate of the mean of any bounded function with fixed tolerance.
\begin{defn}
\label{def:stat}
 Let $D$ be a distribution over a domain $X$ and $\tau >0$. A statistical query oracle $\STAT_D(\tau)$ is an oracle that given as input any function $\phi : X \rightarrow [0,1]$, returns some value $v$ such that
 $|v - D[\phi]| \leq \tau$.
\end{defn}
A special case of statistical queries are counting or linear queries in which the distribution $D$ is uniform over the elements of a  given database $S \in X^n$. In other words the goal is to estimate the empirical mean of $\phi$ on the given set of data points. This setting is studied extensively in the literature on differential privacy (see \citep{DworkRoth:14} for an overview) and our discussion applies to this setting as well.

Tolerance $\tau$ of statistical queries roughly corresponds to the number of random samples in the traditional setting. Namely, the Chernoff-Hoeffding bounds imply that $n$ i.i.d.~samples allow estimation of $D[\phi]$ with tolerance $\tau = \Theta(1/\sqrt{n})$ (with high probability). However, if the variance of $\phi(\bx)$ is low then $n$ random samples will likely give a more accurate estimate. To address this discrepancy a somewhat stronger oracle was introduced in \citep{FeldmanGRVX:12}.
\begin{defn}
\label{def:vstat}
Let $D$ be a distribution over a domain $X$ and $n >0$. A statistical query oracle $\VSTAT_D(n)$ is an oracle that given as input any function $\phi : X \rightarrow [0,1]$ returns a value $v$ such that $|v -  p| \leq \max\left\{\frac{1}{n}, \sqrt{\frac{p(1-p)}{n}}\right\}$, where $p \doteq D[\phi]$.
\end{defn}

An important property of the accuracy guarantees of the $\VSTAT$ oracle is that the guarantees can be easily implemented in the model which allows only a single bit to be communicated about each sample.  Formally, a 1-bit sampling oracle for $D$ is defined as follows \citep{Ben-DavidD98}.
\begin{defn}[$\COMM$ oracle]
\label{def:1stat}
  Let $D$ be a distribution over the domain $X$. The $\COMM_D$ oracle is the oracle that given any function $h: X \rar \zo$,   takes an independent random sample $\bx$ from $D$ and returns $h(\bx)$.
\end{defn}
Identical and closely related models are often studied in the context of distributed statistical estimation with communication constraints  (\eg \citep{luo2005universal,rajagopal2006universal,ribeiro2006bandwidth,ZhangDJW13,SteinhardtD15,SteinhardtVW16,suresh2016distributed}).

As shown by \citet{FeldmanGRVX:12}, to simulate $\VSTAT_D(n)$ on a query $\phi : X \rightarrow [0,1]$, one can take the mean of $O(n)$ bits where each bit is equal to an independent coin flip with bias $\phi(\bx)$ for a random sample $\bx\sim D$. To do this, for each sample, we define a Boolean function $\phi_\theta$ whose value at a point $x$ is the indicator of the condition $\phi(x) \leq \theta$, where $\theta$ is chosen uniformly from $[0,1]$ and then pass $\phi_\theta$ to the $\COMM_D$ oracle. This implies that any algorithm that works in the SQ model with $\VSTAT$ can be simulated using the $1$-bit sampling oracle (see Thm.~\ref{th:stat-from-unbiased} for the formal statement).
The correspondence between these oracles is quite tight since $n$ samples from the $1$-bit sampling oracle can be simulated (with high probability) using $n$ queries to $\VSTAT_D(O(n))$ \citep{FeldmanGRVX:12}\footnote{Weaker version of this correspondence for $\STAT$ was first shown by \citet{Ben-DavidD98} and one for a more general communication model by \citet{SteinhardtVW16}.
}. Hence for most of the discussion we will focus on results for the $\VSTAT$ oracle.

\subsection{The problem}
Here we address the case of estimating the expectation of a general real-valued $\phi$. Specifically, we want to design an algorithm that given access to $\VSTAT_D(n)$ estimates the expectation of $\phi$ (almost) as well as possible using $n$ random samples of $\phi(\bx)$. The main that we need to deal with is that the accuracy of estimates in the restricted models we consider depends linearly on the range of $\phi$. In contrast, the standard deviation of the mean of $\phi$ on $n$ i.i.d.~samples is equal to $\sqrt{(D[\phi^2]-D[\phi]^2)/n}$ and does not depend on the range of $\phi$ (which can even be infinite). Equivalently, for a given $\eps >0$, we want to estimate the expectation within $\eps$ with estimation complexity that is close to the sample complexity of this task. Here the {\em estimation complexity} of a statistical query algorithm using $\VSTAT_D(n)$ refers to the value $n$; for an algorithm using $\STAT_D(\tau)$ it is $1/\tau^2$.

While mean estimation is the most basic statistical problem, we are not aware of any treatment of the general case for either the SQ or 1-bit sampling model. The closest literature we are aware of considers the problem of mean estimation from sensor networks. In this case each of the sensors can communicate a few bits but does not receive feedback. This corresponds to asking queries non-adaptively in our model. Most of this literature operates under substantially stronger assumptions such as knowing the distribution $D$ up to a shift (or, equivalently, having an unbiased noise from a fixed distribution added to some unknown value) (\eg \citep{ribeiro2006bandwidth}). In more general treatments we are aware of a bound on the range is assumed with linear dependence of the error on this bound (\eg \citep{luo2005universal}). As a result both the nature of algorithms and the resulting bounds appear to be quite different from our setting.

Another useful property of the $\VSTAT$ oracle is that known lower bound techniques against SQ algorithms using $\STAT_D(\tau)$ also apply to the stronger $\VSTAT_D(1/\tau^2)$ oracle. As a result lower bounds for $\VSTAT$ usually correspond more tightly to the known computational and sample complexity bounds on the problem (\eg \citep{FeldmanGRVX:12,FeldmanPV:13,FeldmanGV:15,DiakonikolasKS:16}). One could also easily define a stronger statistical query oracle whose accuracy scales with the standard deviation of $\phi(\bx)$. However, it is unknown how to obtain lower bounds against this stronger oracle using existing analysis techniques (\eg \citep{FeldmanGRVX:12,FeldmanPV:13,Feldman:16sqd}). From this point of view, our goal is to show that such lower bounds can be obtained indirectly by simulating a stronger oracle using the easier to analyze $\VSTAT$ oracle.

The first approach to this general case and the one most commonly used is simply to scale/shift the range of the function $\phi$ to $[0,1]$ and then just ask a single query to $\VSTAT_D(n)$ (the answer can then be scaled/shifted back). However that will result in error that is at least $R/n$, where $R=\sup_{x\in X} \phi(x)- \inf_{x\in X} \phi(x)$ (and in many cases the error will be $\Omega(R/\sqrt{n})$).

Another natural approach would be to truncate $\phi$ to be in some range $[-R,R]$ such that $|D[\phi] - D[\phi']|\leq \eps$, where $\phi'$ denotes the truncated function. The expectation of  $\phi'$ can then be estimated by scaling it to $[0,1]$ as before. This approach would provide a reasonably accurate estimate when $\phi(\bx)$ is strongly concentrated (\eg sub-gaussian), a good estimate of the standard deviation is known and we know the mean up to a constant multiple of standard deviations (an example of such analysis appears in \citep{FeldmanGV:15}). But there are many common settings of $\phi$ and $D$ where these assumptions would not apply.

\subsection{Our contribution}

We show that it is possible to deal with the case when no upper bound on variance is known and we are only given a loose upper bound on the second moment of $\phi[\bx]$. We prove that access to $\VSTAT_D(n)$ suffices to estimate $D[\phi]$ within $\tilde O(\sigma/\sqrt{n})$, where $\sigma = \sqrt{D[\phi^2]-D[\phi]^2}$ is the standard deviation of $\phi(\bx)$. Thus $\VSTAT_D(n)$ can be used to obtain an estimate with accuracy comparable to the accuracy implied by the Bernstein concentration inequality for a sum of i.i.d.~random variables combined with the median of means technique (without the median step the accuracy will not be range-independent with high probability).
The number of queries needed to implement this algorithm is logarithmic in the upper bound on the second moment and $n$.
\begin{thm}
\label{thm:est-variance-n-intro}
There exists a statistical query algorithm that given an integer $n$, $\zeta>0$, $B>0$, $\phi:X \rar \R$ such that $D[\phi^2]\leq B^2$ and  access to  $\VSTAT_D(n)$, outputs a value $v$ such that $|D[\phi]-v| \leq \frac{8\sigma\log(8 n)}{\sqrt{n}} + \zeta$, where $\sigma^2 \doteq D[\phi^2] - D[\phi]^2$. The algorithm uses $3\log(4nB/\zeta^2)$ queries.
\end{thm}

The algorithm is based on several relatively simple ideas. First, the bound on the moment implies that for a suitable choice of quantiles of the distribution of $\phi(\bx)$ we can truncate the range of $\phi$ to the values of these quantiles without changing the mean of $\phi$ too much. We show that $\VSTAT(n)$ oracle allows to find the necessary quantiles. This step requires discretizing the range of $\phi$ and contributes the additional $\zeta$ term in the error. Note that we cannot eliminate this error term without knowing a lower bound on $\sigma$ but the dependence of the number of queries on $1/\zeta$ is just logarithmic. Further, we show that after this truncation, the length of the range of the function becomes $O(B/\eps)$. We then describe an algorithm whose error scales linearly with the square root of the second moment of $\phi(\bx)$ instead of the square root of the mean as is guaranteed by $\VSTAT$. The algorithm reduces the computation of the mean of $\phi$ to multiple computations of a mean of $\phi$ conditioned on being in an interval of constant dynamic range. A logarithmic number of such intervals are sufficient to cover the range of $\phi$ after the truncation\footnote{We remark that this step is different from computing the expectation of each of the (most-significant) bits of the value of $\phi$.}.  Finally, to go from an error bounded in terms of the second moment to an error bounded in terms of the standard deviation we show that after shifting the function $\phi$ by the value of an approximate median, the second moment will be upper-bounded by a constant multiple of the variance.

An immediate corollary of this result and Theorem \ref{th:stat-from-unbiased} is the following upper bound on the accuracy that can be obtained  using $n$ queries to 1-bit sampling oracle.
\begin{thm}
\label{thm:est-variance-n-intro}
There exists an algorithm that given $B>0,\delta>0$, $\zeta>0$, $\phi:X \rar \R$ such that $D[\phi^2]\leq B^2$ and access to $n$ queries to $\COMM_D$, with probability $\geq 1-\delta$, outputs a value $v$ such that $|D[\phi]-v| = \frac{\sigma \cdot \tilde O\left(\log n \cdot \log (nB/\zeta) \cdot \log(1/\delta)\right)}{\sqrt{n}} + \zeta$, where $\sigma^2 \doteq D[\phi^2] - D[\phi]^2$. \end{thm}

We also give a simpler algorithm and bounds for the case when we have an upper bound on $\sigma$ and/or already have an estimate of $D[\phi]$ within $O(\sigma)$ (see Lemma \ref{lem:est-positive-2moment-knownB}). In this case the algorithm's queries are non-adaptive, that is, query functions do not depend on answers to previous queries. Adaptivity of queries is particularly important for the distributed setting that motivates the 1-bit sampling model since queries need to be communicated back to clients that hold the samples. This step is likely to be slow or even impossible in some systems (such as the sensor systems considered in \citep{luo2005universal}).

We remark that while the bounds we give are stated primarily for $\VSTAT_D(n)$, they imply identical upper bounds for $\STAT_D(1/n)$. However $\STAT_D(1/n)$ has estimation complexity of $n^2$ and better bounds for it might be possible in some cases. We also note that we have not attempted to optimize the constants or even the log factors for the 1-bit sampling  case. Deriving tighter bounds for this case directly is a natural open problem.

\subsection{Applications}
As an example application we consider the problems of mean vector estimation and stochastic optimization in the SQ model \citep{FeldmanGV:15}. In the $\ell_2$ mean vector estimation problem the goal is to estimate $\bar{x} \doteq \E_{\bx\sim D}[\bx]$ within $\eps$ in the $\ell_2$ norm, where $D$ is a distribution over vectors in $\R^d$. This problem is central to implementation of gradient-based (or first-order) optimization methods in the stochastic setting. In \citep{FeldmanGV:15} it was showed that using either a randomly rotated basis or Kashin's representation \citep{Lyubarskii:2010} this problem can be solved using $O(d)$ queries to $\STAT_D(\Omega(\eps))$ when $D$ is supported on the unit ball. A related technique is also used in \citep{suresh2016distributed} for a low communication model.
Our results give a simpler and more general algorithm for the problem using the $\VSTAT_D(\tilde O(1/\eps^2))$ oracle or queries to $\COMM_D$.
\begin{cor}
\label{cor:est-variance-vector-eps-intro}
There exists a statistical query algorithm that given $\eps \in (0,1)$ and $B>\eps$, for any distribution $D$ over $\R^d$ such that $\E_{\bx\sim D}[\|\bx\|_2^2] \leq B^2$ and $\E_{\bx\sim D}[\|\bx-\bar{x}\|_2^2] \leq 1$ outputs a vector $\hat{x}$ such that $\|\hat{x}-\bar{x}\|_2 \leq \eps$, where $\bar{x} \doteq \E_{\bx\sim D}[\bx]$.
The algorithm uses $O(d(\log(dB/\eps)))$  queries to $\VSTAT_D(O((\log(1/\eps)/\eps)^2))$.
\end{cor}
High dimensional mean estimation is the key component of the SQ algorithms for stochastic convex optimization and learning of linear threshold functions. Specifically, they are used to estimate the gradient of the expected objective function in a variety of gradient descent-based algorithms. Without constraints on data access, it is well-known that stochastic gradient descent will achieve low expected error as long as the variance of the gradients is bounded (\eg \citep{Nemirovski:2009}). It is natural to ask whether this case can also be handled in the SQ and 1-bit sampling models. Corollary \ref{cor:est-variance-vector-eps} immediately implies that the optimization algorithms in \citep{FeldmanGV:15} can be extended from the setting in which the (sub-)gradients of all the functions in the support of the input distribution are uniformly upper-bounded to the setting in which only variance of (sub-)gradient vectors is upper-bounded (additional details can be found in \citep{FeldmanGV:15}).
These corollaries can in turn be easily translated to the 1-bit sampling model and we describe an example of such corollary in Section \ref{sec:apps}.


\section{Preliminaries}
\label{sec:prelims}
For integer $n\geq 1$ let $[n]\doteq \{1,\ldots, n\}$. Random variables are denoted by bold
letters, e.g., $\bx$. We denote the indicator function of an event $A$
(i.e., the function taking value zero outside of $A$, and one on $A$) by $\ind{A}$.

We will use the following notation for a truncation operation. For a real value $z$ and $a \in \R^+$, let
 \[
m_a(z):=\left\{
\begin{array}{rl}
z & \mbox{if } |z| \leq a\\
a			     & \mbox{if } z> a\\
-a			     & \mbox{if } z<-a.
\end{array}
\right.
\]
Let $r_a(z) \doteq z - m_a(z)$.

We consider mean estimation with two related oracles. The first one is the $\VSTAT$ oracle from \citep{FeldmanGRVX:12} (see Def.~\ref{def:vstat}) that strengthens the $\STAT_D$ oracle of \citet{Kearns:98} (see Def.~\ref{def:stat}).
Clearly, $\VSTAT_D(n)$ is at least as strong as $\STAT_D(1/\sqrt{n})$ (but no stronger than $\STAT_D(1/n)$).
When the range of $\phi$ is $\zo$, one way to think about $\VSTAT$ is as providing a confidence interval for the bias $p$, namely $[v - \tau_v,v +\tau_v]$, where $\tau_v \approx \max\{1/n,\sqrt{(v(1-v)/n}\}$. The accuracy $\tau_v$ that $\VSTAT$ ensures corresponds (up to a small constant factor) to the width of the standard confidence interval (say, with 95\% coverage) for the bias $p$ of a Bernoulli random variable given $n$ independent samples (\eg Clopper-Pearson interval \citep{ClopperP1934}). Therefore, at least for Boolean queries, it captures precisely the accuracy that can be inferred when estimating the mean using $n$ random samples. 
For more general real-valued queries (even those with range in $[0,1]$) the accuracy of $\VSTAT_D(n)$ might no longer reflect the accuracy that will be achieved given $n$ i.i.d.~samples.

The second oracle we consider is the 1-bit sampling oracle (see Def.~\ref{def:1stat}). In the context of learning theory it was first studied by \citet{Ben-DavidD98}. Our results for this oracle will be obtained via the following simulation of the $\VSTAT_D$ oracle using the $\COMM_D$ oracle.
\begin{thm}[\citep{FeldmanGRVX:12}]
\label{th:stat-from-unbiased}
Let  $n,q>0$ be any integers and $\delta>0$. For any algorithm $\A$ that asks at most $q$ queries to $\VSTAT_D(n)$ there exists an algorithm $\B$ that provides answers that satisfy the accuracy guarantees of  $\VSTAT_D(n)$ to all the queries of $\A$ with probability at least $1-\delta$.
 $\B$ uses $O(qn\cdot \log(q/\delta))$ queries to $\COMM_D$.
\end{thm}


\section{Mean estimation using statistical queries}
\subsection{Non-adaptive algorithm}
We start by showing how the accuracy of mean estimation can be improved given a reasonably tight bound on the second moment of $\phi$. The algorithm for this problem has the advantage of being non-adaptive.

In the first step we show that the dependence on $\sqrt{D[\phi]}$ in the accuracy guarantees of $\VSTAT$ can be strengthened to $\sqrt{D[\phi^2]}$ at the expense of logarithmic factors in the complexity. We will use $\log$ to denote the logarithm to base 2.

\remove{
\begin{lem}
\label{lem:est-via-integral}
There exists a statistical query algorithm that given positive integers $m,n$, for any function $\phi:X\rar[0,R]$, where $R\geq 1$ and any input distribution $D$ over $X$ $[0,R]$ outputs a value $v$ such that $|D[\phi]-v| \leq \frac{s\ln(eR/s)}{\sqrt{n}} + R/n + R/m$, where $s \doteq \sqrt{D[\phi^2]}$.
The algorithm uses $m$ (non-adaptive) queries to $\VSTAT_D(n)$.
\end{lem}
\begin{proof}
The estimation of $D[\phi]$ relies on the following simple fact
$$D[\phi] = \int_0^R \pr_D[\phi(\bx) \geq t]dt .$$
As a first step we approximate this integral by a finite sum. Using the fact that $\pr_D[\phi(\bx) \geq t]$ is a monotone decreasing function of $t$ with range in $[0,1]$ we obtain that
\equ{\left| \int_0^R \pr_D[\phi(\bx) \geq t]dt - \frac{R}{m}\cdot \sum_{i\in [m]} \pr_D[\phi(\bx) \geq iR/m] \right| \leq \frac{R}{m} .\label{eq:int-approx}}

For every $t \geq 0$, we can use query function $\ind{\phi(x) \geq t}$ to $\VSTAT_D(n)$ to get an estimate $\hat{p}_t$ of $p_t \doteq \pr_D[\phi(\bx) \geq t]$ such that
$$|\hat{p}_t - p_t| \leq \max\left\{\frac{1}{n},\sqrt{\frac{p_t}{n}}\right\} \leq \frac{1}{n}+\sqrt{\frac{p_t}{n}}
\leq \frac{1}{n} + \min\left\{1,\frac{s}{t}\right\} \cdot \frac{1}{\sqrt{n}},$$ where, we used that, by Chebyshev's inequality, $p_t
\leq \left\{1,\frac{D[\phi^2]}{t^2}\right\} =
\min\left\{1,\frac{s^2}{t^2}\right\}$.

Using such an estimate $\tilde{p}_t$ for every $t\in\{ iR/m \cond i\in [k]\}$ to approximate the sum from eq.~\eqref{eq:int-approx} we get
\alequn{\left|\frac{R}{m}\cdot \sum_{i\in [m]} \pr_D[\phi(\bx) \geq iR/m] - \frac{R}{m}\cdot \sum_{i\in [m]} \tilde{p}_{iR/m} \right| & \leq \frac{R}{m}\cdot \sum_{i\in [m]} \left| p_{iR/m}-  \tilde{p}_{iR/m}\right| \\
& \leq \frac{R}{m} \cdot \left(\frac{m}{n} +  \frac{1}{\sqrt{n}} \cdot \sum_{i\in [m]} \min\left\{1,\frac{s}{iR/m}\right\} \right)\\
& \leq \frac{R}{n} + \frac{1}{\sqrt{n}} \cdot \int_0^R \min\left\{1,\frac{s}{t}\right\} dt \\
&= \frac{R}{n} +  \frac{s+ s\ln (R/s)}{\sqrt{n}}. }

Together with eq.~\eqref{eq:int-approx} this gives the claim.
\end{proof}
}
\begin{lem}
\label{lem:est-via-log}
There exists a statistical query algorithm that, for $n,R >0$, any function $\phi:X\rar[0,R]$ and any input distribution $D$ over $X$, outputs a value $v$ such that $|D[\phi]-v| \leq \frac{4R}{n} + \frac{2 s \cdot \log n}{\sqrt{n}}$, where $s \doteq \sqrt{D[\phi^2]}$.
The algorithm uses at most $\log n$ (non-adaptive) queries to $\VSTAT_D(n)$.
\end{lem}
\begin{proof}
We assume for simplicity that $R=1$ since we can always scale $\phi$ to this setting and then scale back the result.
We let $t = \lfloor \log n \rfloor$ and observe that
\equ{D[\phi] = \E_D\lb\phi(\bx) \cdot \ind{\phi(\bx) \in [0,2^{-t}]}\rb + \sum_{i\in [t]} \E_D\lb\phi(\bx) \cdot \ind{\phi(\bx) \in (2^{-i},2^{-i+1}]}\rb . \label{eq:decompose-phi-log}}

For every $i \in [t]$, we define $\phi_i$ to be the restriction of $\phi$ to values in the interval $(2^{-i},2^{-i+1}]$, scaled and shifted to the range $[0,1]$. Namely
$$\phi_i(x) \doteq 2^{i-1} \cdot \phi(x) \cdot  \ind{\phi(\bx) \in (2^{-i},2^{-i+1}]}  .$$
Note that for every $x$, $\phi_i(x) \in [0,1]$. Using this definition, we can rewrite eq.\eqref{eq:decompose-phi-log} as
\equ{D[\phi] = \E_D\lb\phi(\bx) \cdot \ind{\phi(\bx) \in [0,2^{-t}]}\rb + \sum_{i\in [t]} 2^{-i+1} \cdot D[\phi_i]. \label{eq:decompose-phi-log-simple}}

For every $i \in [t]$, we query function $\phi_i$ to $\VSTAT_D(n)$ to get an estimate $v_i$ of $D[\phi_i]$.
By Chebyshev's inequality, $$D[\phi_i] \leq \pr_D[\phi(\bx)> 2^{-i}] \leq \frac{D[\phi^2]}{2^{-2i}} =
2^{2i} \cdot s^2 .$$ Therefore, by the definition of $\VSTAT_D(n)$,
\equ{ |v_i -D[\phi_i]| \leq \max\left\{\frac{1}{n},\sqrt{\frac{D[\phi_i]}{n}}\right\} \leq \frac{1}{n}+\sqrt{\frac{D[\phi_i]}{n}}
\leq \frac{1}{n} + \frac{2^i s}{\sqrt{n}}. \label{eq:approx-i}}

Let $v \doteq   \sum_{i\in [t]} 2^{-i+1} \cdot v_i$.

Then, by eq.~\eqref{eq:decompose-phi-log-simple} and eq.~\eqref{eq:approx-i} we get that

\alequn{| D[\phi] - v|  &= \left| \E_D\lb\phi(\bx) \cdot \ind{\phi(\bx) \in [0,2^{-t}]}\rb + \sum_{i\in [t]} 2^{-i+1} \cdot (D[\phi_i]-v_i) \right| \\
& \leq  \E_D\lb\phi(\bx) \cdot \ind{\phi(\bx) \in [0,2^{-t}]}\rb + \sum_{i\in [t]} 2^{-i+1} \cdot \left| D[\phi_i]-v_i \right| \\
&\leq 2^{-t} + \sum_{i\in [t]} 2^{-i+1} \cdot \left( \frac{1}{n} + \frac{2^i s}{\sqrt{n}} \right) \\
&\leq \frac{2}{n} + \frac{2}{n} + \frac{2t \cdot s}{\sqrt{n}} \leq  \frac{4}{n} + \frac{2s \cdot \log n }{\sqrt{n}}.
}
\end{proof}

\remove{
\begin{remark}
This lemma can also be seen as a strengthening of $\VSTAT$ oracle. For $\phi: X \rar [0,1]$,
$\VSTAT_D(n)$ gives an estimate within $\max\left\{\frac{1}{n}, \sqrt{\frac{p}{n}}\right\}$ which is roughly $\sqrt{\frac{p}{n}} + \frac{1}{n}$. On the other hand, the lemma implies that using $\log n$ queries to $\VSTAT_D(n)$ oracle we can get an estimate within
roughly $\frac{s\log n}{\sqrt{n}} +  \frac{3}{n}$. Clearly $s = \sqrt{D[\phi^2]} \leq \sqrt{p}$ and therefore, ignoring the logarithmic factors, the resulting bound improve $\VSTAT$ oracle (up to a $O(\ln(1/p))$ factor). \remove{Also note that the algorithm above requires only Boolean ($\{0,1\}$-valued) query functions and therefore implies a reduction from real-valued queries to Boolean ones.}
\end{remark}
}

In the algorithm above the parameter of $\VSTAT$ scales linearly with the range $R$. However given a bound on $D[\phi^2]$ we can truncate the range of $\phi$ with little change in expectation. We use this in the lemma below.

\begin{lem}
\label{lem:est-positive-2moment-knownB}
There exists a statistical query algorithm that given $B>0$ and  $\eps \in (0,B/16]$, for any distribution $D$ and function $\phi:X\rar \R^+$  such that $D[\phi^2] \leq B^2$ outputs a value $v$ such that $|D[\phi]-v| \leq \eps$.
The algorithm uses at most $3\log(B/\eps)$ (non-adaptive) queries to $\VSTAT_D((8B \log(B/\eps)/\eps)^2)$.
\end{lem}
\begin{proof}
We first assume that $B=1$. Observe that we can truncate the range of $\phi$ to $[0,a]$ for $a\doteq 4/\eps$ without a significant change in the expectation. Namely,
we claim that $|D[m_a(\phi)] - D[\phi] | \leq \eps/4$. To see this when $a= 4/\eps$ note that $\phi-m_a(\phi) = r_a(\phi)$ and
$$\E[r_a(\phi)] = \int_0^\infty \pr[r_a(\phi)\geq t] dt = \int_0^\infty \pr[\phi\geq t+a] dt = \int_a^\infty \pr[\phi\geq t] dt \leq \int_a^\infty \frac{1}{t^2}dt =\frac{1}{a} \leq \eps/4,$$ where we used Chebyshev's inequality to obtain a bound on $\pr[\phi\geq t]$.

Now to estimate $D[m_a(\phi)]$ we use Lemma \ref{lem:est-via-log} with $n=(4\log(1/\eps)/\eps)^2$.  For the given range and our assumption that $D[\phi^2]\leq 1$, this leads to an error of at most $$\frac{16}{(8\log(1/\eps)/\eps)^2\cdot \eps} + \frac{2\log(8\log(1/\eps)/\eps)}{8\log(1/\eps)/\eps} \leq \frac{\eps}{4\log(1/\eps)^2}  + \frac{\eps \cdot \log(8\log(1/\eps)/\eps)}{4\log(1/\eps)} \leq \frac{3\eps}{4},$$ where we used the assumption that $\eps \leq 1/16$ to obtain the last inequality. Altogether this implies that the output of the lemma on the truncated function will have an error of at most $\eps$. Finally, to generalize the analysis to any $B>0$ we simply scale the random variable by $1/B$ and estimate its mean within $\eps/B$.
\end{proof}

\subsection{Quantile estimation}
To deal with the case where we do not have a tight upper bound on the second moment we will need the following lemma about the estimation complexity of finding approximate distribution quantiles using the $\VSTAT$ oracle. The algorithm itself is just the folklore algorithm for finding an (approximate) quantile of a distribution using binary search.
\begin{lem}
\label{lem:percentile}
Let $\phi$ be a function with a finite range $Z \subset \R$.  There exists an SQ algorithm that given $p$ and $\delta$ such that $1 \geq p\geq 2\delta>0$, outputs a point $a \in Z$ such that $\pr_D[\phi \geq a] \geq p-\delta$ and $\pr_D[\phi > a] < p$. The algorithm uses $\log(|Z|)$ queries to $\VSTAT_D(4p/\delta^2)$.
\end{lem}
\begin{proof}
For any $z\in Z$, let $p_z \doteq \pr_D[\phi \geq z]$.
We perform a binary search to find the largest point $z$ such that the estimate of $\pr_D[\phi \geq z]$ given by
$\VSTAT_D(4p/\delta^2)$ is at least $p-\delta/2$. We denote the point by $a$ and refer to estimates we have obtained by $\tilde{p}_z$.
By definition, $\tilde{p}_a \geq p-\delta/2$ and therefore $p_a \geq p-\delta$, since otherwise $$\tilde{p}_a < p-\delta + \max\left\{ \sqrt{\frac{p-\delta}{4p/\delta^2}},\frac{\delta^2}{4p}\right\} < p-\delta+\delta/2 = p-\delta/2.$$ On the other hand for the smallest point in $Z$ larger than $a$ (denote it by $a_+$) we know that $\tilde{p}_{a_+} < p-\delta/2$. This implies that $p_{a_+} < p$ since otherwise
$$\tilde{p}_{a_+} \geq p_{a_+} - \max\left\{ \sqrt{\frac{p_{a_+}}{4p/\delta^2}},\frac{\delta^2}{4p}\right\} = p_{a_+}- \frac{\delta}{2}\cdot \sqrt{\frac{p_{a_+}}{p}} \geq \sqrt{p\cdot p_{a_+}} - \delta/2 \geq p-\delta/2,   $$
where in the penultimate step we multiplied the inequality by $\sqrt{\frac{p}{p_{a_+}}}$ which is at most $1$ by our contrapositive assumption.
This means that $\pr_D[\phi > a]= p_{a_+} < p$.
\end{proof}
As a special case we obtain the following corollary.
\begin{cor}
\label{cor:tail-percentile}
Let $\phi$ be a function over $X$ with a finite range $Z \subset \R$.  There exists an SQ algorithm that given an integer $n$ outputs a point $a \in Z$ such that $\pr_D[\phi \geq a] \geq 8/n$ and $\pr_D[\phi > a] < 16/n$. The algorithm uses $\log(|Z|)$ queries to $\VSTAT_D(n)$.
\end{cor}
One natural way to apply this lemma in the continuous setting is to first discretize the range with some step size $\zeta$ and then use Lemma \ref{lem:percentile}. This leads to the following version of Lemma \ref{lem:percentile}:
\begin{cor}
\label{cor:percentile-cont}
Let $\phi:X\rar [-B,B]$ be a function and $\zeta >0$.  There exists an SQ algorithm that given $p$ and $\delta$ such that $1 \geq p\geq 2\delta>0$, outputs a point $a$ that is an integer multiple of $\zeta$ such that $\pr_D[\phi \geq a] \geq p-\delta$ and $\pr_D[\phi \geq a+\zeta] < p$. The algorithm uses $\lceil \log(2B/\zeta) \rceil$ queries to $\VSTAT_D(4p/\delta^2)$.
\end{cor}

\subsection{General case}
Lemma \ref{lem:est-positive-2moment-knownB} has estimation complexity that scales linearly with the bound on the second moment and is the strongest statement we give without relying on adaptive queries. Next, using Lemma \ref{lem:percentile}, we describe a procedure that can estimate the mean with just logarithmic dependence on a given upper bound on the second moment.
\begin{lem}
There exists a statistical query algorithm that given an integer $n$, $\zeta>0$, $B>0$ and $\phi:X \rightarrow \R^+$ such that
 $D[\phi^2]\leq B^2$, outputs a value $v$ such that $|D[\phi]-v| \leq \frac{2s\log (8n)}{\sqrt{n}} + \zeta$, where $s \doteq \sqrt{D[\phi^2]}$.
The algorithm uses $\log(4Bn/\zeta^2)$ queries to $\VSTAT_D(n)$.
\end{lem}
\begin{proof}
We start by truncating and discretizing the range of $\phi$, namely let $\psi$ be equal to $\phi$ rounded down to the closest multiple of $\zeta/2$ and truncated at $2B/\zeta$.
As in Lemma \ref{lem:est-positive-2moment-knownB}, we note that the condition $D[\phi^2]\leq B^2$ implies that truncation to range $[0,2B/\zeta]$ step can affect the expectation by at most $\zeta/2$. Clearly, rounding down to the closest multiple of $\zeta/2$ also affects the expectation by at most $\zeta/2$. This means that $\psi$ has range $Z$ of size at most $4B/\zeta^2$,  $|D[\phi]-D[\psi]| \leq \zeta$ and $D[\psi^2] \leq s^2$.

We now further truncate the range of $\psi$ to exclude values that are in the top $8/n$ quantile. Namely, by Corollary \ref{cor:tail-percentile}, we can find a value $a\in Z$ such that  $\pr_D[\psi \geq a] \geq 8/n$ and $\pr_D[\psi > a] < 16/n$. Let $\psi_a(x)$ be defined as $m_a(\psi(x))$.

We first observe that $s^2 = D[\psi^2] \geq a^2 \cdot 8/n$. Therefore $a \leq \frac{s\sqrt{n}}{\sqrt{8}}$. Next note that
\alequ{
|D[\psi] - D[\psi_a]|& \leq \sum_{z\in Z,\ z>a} z \cdot \pr_D[\psi = z] \nonumber \\
& \leq \sqrt{\sum_{z\in Z,\ z>a} \pr_D[\psi = z]} \cdot \sqrt{\sum_{z\in Z,\ z>a} z^2\cdot \pr_D[\psi = z]} \nonumber \\
& \leq \sqrt{\pr_D[\psi > a]} \cdot \sqrt{D[\psi^2]} \leq \frac{4 s}{\sqrt{n}}  \label{eq:trunc-approx},}
where we used the Cauchy-Schwartz inequality to obtain the second line.

We can now apply Lemma \ref{lem:est-via-log} to $\psi_a$ and obtain a value $v$ such that
$$\left|D[\psi_a]-v \right| \leq \frac{2s\log n}{\sqrt{n}} + \frac{4a}{n} \leq \frac{2s\log n + \sqrt{2}s}{\sqrt{n}}.$$
Combining this with eq.~\eqref{eq:trunc-approx} and the properties of $\psi$ we get the claim.
\end{proof}

We can easily extend these results to variables with possibly negative range by estimating the mean in the positive and the negative range separately. Further, one does not necessarily need to split the range at $0$. Any value $a$ can be used and the resulting bound will be in terms of $s_a \doteq \sqrt{D[(\phi-a)^2]}$ instead of $s_0 = \sqrt{D[\phi^2]}$. Naturally, in order to reduce the error we should use $a=D[\phi]$ which would give an error in terms of variance $\sigma^2 = D[(\phi-a)^2] - D[\phi]^2$. The true mean is not known to us but, as we show below, we can use an (approximate) median instead.
\begin{lem}
Let $\bz$ be a random variable over $\R$ and let $a$ be any point such that $\pr[\bz \geq a] \geq 1/3$ and $\pr[\bz \leq a] \geq 1/3$. Then $\E[(\bz-a)^2] \leq 4 (\E[\bz^2] - \E[\bz]^2)$.
\end{lem}
\begin{proof}
Let $\bar{z} \doteq \E[\bz]$ and $\sigma^2 \doteq \E[(\bz-\bar{z})^2]$. Observe that $$\sigma^2 = \E[(\bz-\bar{z})^2] \geq
(\bar{z}-a)^2 \cdot \pr[\bz \geq |(\bar{z}-a)|] \geq (\bar{z}-a)^2/3 .$$
Now $$ \E[(\bz-a)^2] \leq \sigma^2 + (\bar{z}-a)^2 \leq 4\sigma^2.$$
\end{proof}
Note that for a function with range $Z$ such an approximate median can be found using $\log(|Z|)$ queries to $\VSTAT_D(6)$. Therefore we immediately get the following results:
\begin{thm}
\label{thm:est-variance-n}
There exists a statistical query algorithm that given an integer $n$, $\zeta>0$, $B>0$ and $\phi:X \rar \R$ such that $D[\phi^2]\leq B^2$, outputs a value $v$ such that $|D[\phi]-v| \leq \frac{8\sigma\log(8n)}{\sqrt{n}} + \zeta$, where $\sigma^2 \doteq D[\phi^2] - D[\phi]^2$.
The algorithm uses $3\log(4nB/\zeta^2)$ queries to $\VSTAT_D(n)$.
\end{thm}
An alternative way to state essentially the same result is as follows.
\begin{cor}
\label{cor:est-variance-eps}
There exists a statistical query algorithm that given $B>\zeta >0, \eps \in (0,1)$, for any distribution $D$ and function $\phi:X\rar \R$ such that $D[\phi^2]\leq B^2$ outputs a value $v$ such that $|D[\phi]-v| \leq \eps \sigma + \zeta$.
The algorithm uses $O(\log(B/(\eps\zeta)))$ queries to $\VSTAT_D(O((\log(1/\eps)/\eps)^2))$, where $\sigma^2 \doteq D[\phi^2] - D[\phi]^2$.
\end{cor}

\section{Applications}
\label{sec:apps}
As application we consider the problems of mean vector estimation \citep{FeldmanGV:15}. In the $\ell_2$ mean vector estimation problem the goal is to estimate $\bar{x} \doteq \E_{\bx\sim D}[\bx]$ within $\eps$ in $\ell_2$ norm, where $D$ is a distribution over vectors in $\R^d$.  Our results give a simple and general algorithm for solving this problem using the $\VSTAT$ oracle.
\begin{cor}
\label{cor:est-variance-vector-eps}
There exists a statistical query algorithm that given $\eps \in (0,1)$ and $B>\eps$, for any distribution $D$ over $\R^d$ such that $\E_{\bx\sim D}[\|\bx\|_2^2] \leq B^2$ and $\E_{\bx\sim D}[\|\bx-\bar{x}\|_2^2] \leq 1$ outputs a vector $\hat{x}$ such that $\|\hat{x}-\bar{x}\|_2 \leq \eps$, where $\bar{x} \doteq \E_{\bx\sim D}[\bx]$.
The algorithm uses $O(d(\log(dB/\eps)))$  queries to $\VSTAT_D(O((\log(1/\eps)/\eps)^2))$.
\end{cor}
\begin{proof}
We apply the algorithm given in Thm.~\ref{thm:est-variance-n-intro} to each of the coordinates of $\bx$ with $\zeta = \eps/\sqrt{2 d}$. Namely, for each $i\in [d]$ we use Thm.~\ref{thm:est-variance-n-intro} to estimate the expectation of function $\phi_i(x) = x_i$ and let $\hat{x}_i$ be the result. For every $i \in [d]$, $\E_{\bx\sim D}[\|\bx_i\|_2^2] \leq B^2$ and hence, by setting $n = c (\log(1/\eps)/\eps)^2$ for an appropriately chosen constant $c$, we will obtain that the error in the estimation of coordinate $i$ is at most
$$\frac{\eps}{\sqrt{2}} \cdot \sqrt{ \E_{\bx\sim D}[|\bx_i|^2] - \bar{x}_i^2} +  \frac{\eps}{\sqrt{2 d}}.$$
By observing that $$\sum_{i\in [d]}  \E_{\bx\sim D}[|\bx_i|^2] - \bar{x}_i^2 = \E_{\bx\sim D}[\|\bx-\bar{x}\|_2^2] \leq 1$$ we obtain the claim.
\end{proof}
We remark that if the SQ algorithm $\A$ is non-adaptive then the simulation in Theorem \ref{th:stat-from-unbiased} produces a non-adaptive algorithm. This is particularly useful in the distributed setting since it allows the bits from each of the samples to be communicated in parallel and does not require any communication back to the clients.

These corollaries can in turn be translated to the 1-bit sampling model.

Combining Corollary \ref{cor:est-variance-eps} and Theorem \ref{th:stat-from-unbiased} we obtain the following algorithm for estimating the mean.
\begin{cor}
\label{cor:est-variance-vector-sample}
There exists an algorithm that given $\eps,\delta \in (0,1)$ and $B>\eps$, for any distribution $D$ over $\R^d$ such that $\E_{\bx\sim D}[\|\bx\|_2^2] \leq B^2$ and $\E_{\bx\sim D}[\|\bx-\bar{x}\|_2^2] \leq 1$ outputs a vector $\hat{x}$ such that, with probability at least $1-\delta$, $\|\hat{x}-\bar{x}\|_2 \leq \eps$, where $\bar{x} \doteq \E_{\bx\sim D}[\bx]$.
The algorithm uses $\tilde O(d/\eps^2 \cdot \log(B/(\delta\eps)))$  queries to $\COMM_D$.
\end{cor}

We remark that in the high dimensional setting linear dependence on $d$ is unavoidable for the $1$-bit sampling model \citep{ZhangDJW13} whereas given entire samples, the achievable accuracy is dimension independent. 

\subsection*{Acknowledgements}
I thank the anonymous reviewers of Algorithmic Learning Theory 2017 conference for insightful comments and careful proofreading of this work.


\printbibliography

\end{document}